%% file: main.tex
\documentclass{article}
\usepackage{hyperref}
\usepackage{url}

\usepackage{microtype}
\usepackage{graphicx}
\usepackage{subfigure}
\usepackage{booktabs} 

\usepackage{hyperref}
\usepackage{makecell}

\usepackage{booktabs}       
\usepackage{amsfonts}       
\usepackage{nicefrac}       
\usepackage{microtype}      
\usepackage{amsmath}
\usepackage{algorithm}
\usepackage{algorithmic} 
\usepackage{amsmath,amsfonts,amssymb,bm,amsthm, thm-restate}
\usepackage{commath}
\usepackage{color}
\usepackage{dsfont}

\usepackage[preprint]{neurips_2021}

\usepackage{float}
\usepackage[utf8]{inputenc} 
\usepackage[T1]{fontenc}    
\usepackage{hyperref}       
\usepackage{url}            
\usepackage{booktabs}       
\usepackage{amsfonts}       
\usepackage{nicefrac}       
\usepackage{microtype}      
\usepackage{xcolor}         
\usepackage{tikz}
\usepackage{adjustbox}
\usepackage{thm-restate}

\newcommand*{\rom}[1]{\expandafter\@slowromancap\romannumeral #1@}

\newcommand{\RNum}[1]{\uppercase\expandafter{\romannumeral #1\relax}}

\input{math_commands}

\makeatother
\title{Regularized OFU: an Efficient UCB Estimator for Non-linear Contextual Bandit}

\author{%
  Yichi Zhou \thanks{Microsoft Research, Asia. \texttt{\{yiczho, huishuai.zhang, wche, tyliu\}@microsoft.com}} \And  Shihong Song \thanks{Tsinghua University. bellasong1996@gmail.com, dcszj@mail.tsinghua.edu.cn}
   \And  Huishuai Zhang \footnotemark[1]
   \AND
   Jun Zhu \footnotemark[2]
   \And
   Wei Chen \footnotemark[1] \\
   \And
   Tie-Yan Liu \footnotemark[1] \\

}

\begin{document}

\newtheorem{Lem}{Lemma}
\newtheorem{Thm}{Theorem}
\newtheorem{Defn}{Definition}
\newtheorem{Fact}{Fact}
\newtheorem{Conj}{Conjecture}
\newtheorem{Cor}{Corollary}
\newtheorem{Problem}{Problem}
\newtheorem{State}{Statement}
\newtheorem{Assump}{Assumption}
\newtheorem{Prop}{Proposition}

\maketitle

\begin{abstract}
Balancing exploration and exploitation (EE)  is a fundamental problem in contextual bandit. One powerful principle for EE trade-off is \emph{Optimism in Face of Uncertainty} (OFU), in which the agent takes the action according to an upper confidence bound (UCB) of reward. OFU has achieved (near-)optimal regret bound for linear/kernel contextual bandits. However, it is in general unknown how to derive efficient and effective EE trade-off methods for non-linear  complex tasks, such as contextual bandit with deep neural network as the reward function. In this paper, we propose a novel OFU algorithm named \emph{regularized OFU} (ROFU).  In ROFU, we  measure the uncertainty of the reward by a differentiable function and compute the upper confidence bound by solving a regularized optimization problem. 
We prove that, for  multi-armed bandit, kernel contextual bandit and neural tangent kernel bandit, ROFU achieves (near-)optimal regret bounds with certain  uncertainty measure, which theoretically justifies its effectiveness on EE trade-off. Importantly, ROFU admits a very efficient implementation with gradient-based optimizer, which easily extends to general deep neural network models beyond neural tangent kernel, in sharp contrast with previous OFU methods. The empirical evaluation demonstrates that ROFU works extremely well for  contextual bandits under various settings.
\end{abstract}

\input{contents/introduction}

\input{contents/preliminary}

\input{contents/method}

\input{contents/related_work}

\input{contents/experiment}
\input{contents/conslusion}




\newpage

\appendix
\input{appendix/ucb_proof}
\input{contents/ntkbandit}
\input{appendix/hyperparameter}
\end{document}

%% file: math_commands.tex
\usepackage{amsmath,amsfonts,bm}

\def\eqref#1{equation~\ref{#1}}

\def\1{\bm{1}}

\def\vx{{\bm{x}}}

\def\mI{{\bm{I}}}

\def\mW{{\bm{W}}}

\def\mZ{{\bm{Z}}}

\DeclareMathAlphabet{\mathsfit}{\encodingdefault}{\sfdefault}{m}{sl}
\SetMathAlphabet{\mathsfit}{bold}{\encodingdefault}{\sfdefault}{bx}{n}

\def\gL{{\mathcal{L}}}

\def\gR{{\mathcal{R}}}

\def\sP{{\mathbb{P}}}

\def\sR{{\mathbb{R}}}

\newtheorem{remark}{\textbf{Remark}}

\DeclareMathOperator*{\argmax}{arg\,max}
\DeclareMathOperator*{\argmin}{arg\,min}

%% file: contents/introduction.tex
\section{Introduction}\label{sec:introduction}

Contextual bandit \cite{langford2007epoch,bubeck2012regret} is a basic sequential decision-making problem which is extensively studied and widely applied in machine learning. At each time step in contextual bandit, an agent is presented with a context and the agent chooses an action according to the context. After that the agent will receive a reward conditioned on the context and the selected action. The goal of the agent is to maximize its cumulative reward. One key challenge in contextual bandit is exploration and exploitation (EE) trade-off.   In order to maximize cumulative reward, the agent exploits collected data to take the action with high estimated reward while it also  explores the undiscovered areas to learn knowledge.

\emph{Optimism in Face of Uncertainty} (OFU) is a powerful principle for EE trade-off. When facing uncertainty, OFU first optimistically guesses how good each action could be and then takes the action with highest guess. In philosophy, it exploits collected knowledge by estimating the expected reward and enforces exploration of the areas with large uncertainty via optimism. 
OFU algorithms can be divided into two categories,  i.e., confidence bonus  method \cite{auer2002finite} and confidence set method \cite{jaksch2010near}.  In confidence bonus method, the uncertainty of each action is measured according to its historical frequency, which is added to its mean value as the estimated reward; in confidence set method, the reward function is constrained to a subset with high confidence and the maximum value over the confidence set is taken as the estimated value.   
These OFU algorithms have been proved to achieve the (near-)optimal  regret bounds on tabular and linear problems, including  multi-armed bandit (MAB) \cite{auer2002finite}, linear contextual bandit (LCB) \cite{chu2011contextual} and kernel contextual bandit (KCB) \cite{valko2013finite}, and  have also been successfully applied in practice.  

However, the representation power of the tabular and linear models often becomes insufficient when facing many real-world contextual bandit problems where the scale of the underlying problem  is  huge or even infinite, and the reward mechanism is complex. Therefore, in practice, non-linear deep neural networks (DNN) are used to approximate the reward functions. While DNN models work well to fit training data, it remains challenging to quantify its uncertainty on new data. Therefore, existing OFU methods cannot be efficiently applied to non-linear contextual bandit. For confidence set method, it is difficult to compute the confidence set of the non-linear deep neural networks and hence hard to obtain the optimal action. For the confidence bound method, it is unclear how to design the frequency for non-linear reward functions. Overall, estimating uncertainty of the reward function approximated by non-linear functions, e.g., DNN, becomes a major  obstacle when applying OFU algorithms to real-world applications. 


There are several works \cite{zhou2020neural,filippi2010parametric,zahavy2019deep,gu2021batched} that attempt to generalize OFU to non-linear contextual bandit. One is NeuralUCB \cite{zhou2020neural} that takes the gradient of each observation as the random feature and constructs the upper confidence bound accordingly. It has been shown that NeuralUCB achieves sublinear regret on neural tangent kernel contextual bandit (NTKCB).  However,  at each time step, NeuralUCB requires to compute an inversion of a $p\times p$ matrix where $p$ is the number of model parameters.  As the time complexity of matrix inversion is $O(p^{2.373})$ \cite{davie2013improved}, the computational cost is prohibited for large-scale DNN models. Therefore, it is still largely under-explored to develop OFU algorithms that are able to achieve good  EE trade-off with reasonable computational cost  for non-linear contextual bandit. 


Motivated by these challenges, in this paper, we propose a novel and general OFU method, called \emph{Regularized OFU (ROFU)}. 
In ROFU, we measure the \emph{uncertainty} of the reward function by a differentiable regularizer, and then we calculate the optimistic estimation by maximizing the reward function with the regularizer penalizing its uncertainty. 
 We prove that  ROFU enjoys near-optimal regret on many contextual bandit problems with relatively simple reward mechanism, including MAB, LCB, KCB and NTKCB. 
Importantly, in contrast to existing OFU algorithms, ROFU admits an efficient implementation  with gradient-based optimizer as the optimization problem in ROFU is unconstrained and differentiable, which has a time complexity linear to $p$. Moreover, our algorithm easily extends to general deep neural network models. 

We summarize our contributions as follows:
\begin{itemize}
    
    \item	We show that ROFU achieves (near-)optimal regret on  contextual bandits with simple reward functions, including MAB, LCB, KCB and NTK bandit. 
    These results theoretically justify the efficiency of ROFU on EE trade-off. 
    \item 	
    We propose an algorithm which can efficiently approximate the upper confidence bound with standard gradient-based optimizer. 
    \item  We empirically evaluate ROFU on complex contextual bandits with reward functions beyond linear, kernel and NTK. We show that ROFU  also provides efficient UCB estimation for  popular DNN architectures including MLP, CNN   and ResNet. Moreover,  our algorithm enjoys a smaller regret than strong baselines on real-world  non-linear contextual bandit problems introduced by \cite{deep_bayesian_bandit_showdown}. 
\end{itemize}
\vspace{-.5cm}

%% file: contents/preliminary.tex
\section{Preliminary}\label{sec:preliminary}
In this work, we consider contextual bandit with general reward functions. 

\begin{Defn}[Contextual bandit]\label{eq:contextual_bandit}
Contextual bandit is a sequential decision-making problem where the agent has a set  of actions $A$. At each time step $t$, the agent first observes a context $x_t$, then   selects an action $a_t\in A$ based on the context. After taking the action, the agent receives a (random) reward $r_t$  with $\mathbb{E}r_t:=f_{\theta^*}(x_t,a_t)$ where $f_{\theta^*}(x_t,a_t)$ is a function with unknown groundtruth parameters $\theta^*$.  The agent aims to maximize its expected cumulative reward $\sum_{t\leq T}\mathbb{E}f_{\theta^*}(x_t,a_t)$ which is equivalent to minimizing the regret $R_{T}=\sum_{t\leq T}\max_{a}(f_{\theta^*}(x_t,a)-f_{\theta^*}(x_t,a_t))$. 
\end{Defn}


The most studied contextual bandit is linear contextual bandit, where the reward is generated from a linear function, i.e.,  $f_{\theta^*}(x,a):=\phi(x,a)^\top \theta^*$ where $\phi(x,a)$ is a known feature map. For more complex problems, we have to use non-linear models, such as DNN, to represent the reward function. 

In practice, given dataset $D:=\{(x_t,a_t,r_t)\}_{t\leq |D|}$, we can estimate $\theta^*$ by minimizing a loss function, i.e.,  $\Bar{\theta}:=\argmin_{\theta}\mathcal{L}(\theta;D)$. In this work, we focus on the mean squared loss  and its variants,
\begin{align}
    \text{MSE}(\theta;D):=\frac{1}{|D|}\sum_{t\leq |D|}(f_{\theta}(x_t,a_t)-r_t)^2.
\end{align}
 This is because MSE is a natural choice for regression tasks. Generally speaking, OFU algorithms  trade off exploration and exploitation with an optimistic estimation on the reward for each action, and then select $a_t=\argmax_{a}\text{OFU}(x_t,a)$. 
In the literature, OFU algorithms can be divided into two categories. 
The first one optimistically estimates the reward by adding a confidence bonus to $f_{\bar{\theta}}$, i.e., 
 \begin{align}\label{eq:ofu_bonus}
     \text{OFU}^{B}(x,a):=f_{\bar{\theta}}(x,a) + bonus(x,a).
 \end{align}
We call this OFU method \emph{Bonus-OFU}. In simple cases, the bonus can be constructed using concentration bounds. For example, at time step $\tau$, in UCB1 for MAB  \cite{auer2002finite}, the bonus is $\sqrt{2\log \tau/n(a)}$, where $n(a)$ is the number of pulls on action $a$; in LinUCB for LCB \cite{chu2011contextual}, the bonus is $\sqrt{\phi(x,a)^\top (\sum_{t<\tau} \phi(x_t,a_t)\phi(x_t,a_t)^\top + \mI)^{-1}\phi(x,a)}$. Both UCB1 and LinUCB are proved to have (near-)optimal regret bounds.

The second kind of OFU algorithms optimistically estimates the reward by the best possible value over a confidence set of the reward functions, i.e., 
\begin{align}\label{eq:ofu_constrained}
\text{OFU}^{S}(x,a):=\max_{\theta\in \Theta_{\delta}}f_{\theta}(x,a),
\end{align}
where $\Theta_{\delta}:=\{\theta: \sP(D|\theta)>\delta\}$ and $\sP(D|\theta)$ is the likelihood of $D$ given $\theta$. We call this OFU method \emph{Set-OFU}. 
 
 For convenience, we suppose $\theta$ is a $p$-dimensional vector. 
 Both Bonus-OFU and Set-OFU can effectively trade-off EE, however, they require a lot of computational resources when $f_{\theta}$ is a complex neural network: firstly, existing Bonus-OFU algorithms, such as NeuralUCB \cite{zhou2020neural}, need to calculate the inversion of a $p\times p$ matrix, which is expensive especially for large DNN models; secondly, it is also hard to efficiently compute the confidence set of DNN parameters and even harder to maximize the value over it to compute the $\text{OFU}^{S}$. The computational cost significantly limits the application of OFU methods for complex tasks. 
 
 In the next section, we will propose a new OFU method, named Regularized OFU, which achieves a near-optimal regret bound on simple tasks, and more importantly,   admits a very efficient implementation with a gradient-based optimizer. 
 \vspace{-.5cm}

%% file: contents/method.tex
\section{Method}

Different from Bonus-OFU and Set-OFU, we  measure the uncertainty of reward with a (differentiable) function, and compute the optimistic reward estimation under the regularization of the uncertainty measure function. 
 We call the reward estimation \emph{regularized OFU} (ROFU) and take action accordingly. In round $t$, given context and action pair $(x_t,a)$, ROFU is defined as follows and the algorithm  is presented in Alg. \ref{alg:generalROFU}:

\begin{flalign}
    \hat{\theta}(x_{t},a) &:=\argmax_{\theta} f_{\theta}(x_t,a) - \eta(x_t,a,D_{t-1})\mathcal{R}(\theta;D_{t-1}), \label{eq:rofu_optim}\\
    \text{OFU}^{R}(x_{t},a) &:= f_{\theta_{t-1}}(x_t, a) + g(f_{\hat{\theta}(x_t,a)}(x_t, a)- f_{\theta_{t-1}}(x_t, a)), \label{eq:rofu}
\end{flalign}

\begin{algorithm}[htb!]
\caption{Regularized Optimism in Face of Uncertainty}
\label{alg:generalROFU}  
\begin{algorithmic}[1]

\STATE \textbf{Input}: A reward function $f_{\theta^*}$ with unknown $\theta^*$, number of rounds $T$, a loss function $\mathcal{L}(\theta;D)$, a regularizer $\mathcal{R}(\theta;D)$, a monotonically increasing function $g:\mathbb{R}\rightarrow \mathbb{R}$ with $g(0)=0$ and the coefficient function $\eta(x,a;D)$.
\STATE $D_0:=\emptyset$. 
\FOR{t = 1, ..., T}
    \STATE Observe $x_{t}$. 
    \STATE $\forall a \in A$, compute $\hat{\theta}(x_t,a)$ and $\text{OFU}^R(x_t,a)$ according to Eq. (\ref{eq:rofu_optim}) and (\ref{eq:rofu}). 
    \STATE Take $a_t = \argmax_{a\in A} \text{OFU}^{R}(x_t,a)$ and receive reward $r_t$.
    \STATE Let $D_{t} := D_{t-1}\cup \{(x_t, a_t, r_t)\}$. 
    \STATE Obtain $\theta_t$ by minimizing $\mathcal{L}(\theta;D_t)$.
\ENDFOR 
\end{algorithmic}
\end{algorithm}



where $D_{t-1}$ is the set of collected data in first $t-1$ rounds, $\mathcal{R}$ is a regularizer, $\eta$ is the weight of the regularizer, $g$ is a monotonically increasing function with $g(0)=0$ and $\theta_{t-1}$ is trained by minimizing some objective function, i.e.,  $\theta_{t-1}:=\argmin_{\theta}\mathcal{L}(\theta;D_{t-1})$. We simply take action $a_{\text{ROFU}}:=\argmax_{a\in A}\text{OFU}^{R}(x_t,a)$ to interact with the environment.

 
 In Eq.~(\ref{eq:rofu_optim}), we first solve an unconstrained optimization problem to calculate $\hat{\theta}(x_t,a)$. Then, in Eq.~(\ref{eq:rofu}), we add a bonus to $f_{\theta_{t-1}}(x_t,a)$ according to the difference between $f_{\hat{\theta}(x_t,a)}(x_t,a)$ and $f_{\theta_{t-1}}(x_t,a)$. 
 From Eq. (\ref{eq:rofu_optim}) and (\ref{eq:rofu}), we can see that , while ROFU still adds a bonus on $f_{\theta_{t-1}}(x,a)$ as Bonus-OFU, it does not need to explicitly construct the upper confidence bound via concentration inequalities. Instead, we construct the upper confidence bound by solving a regularized optimization problem which is more  general. In comparison to Set-OFU, our optimization problem is unconstrained, so Eq. (\ref{eq:rofu_optim}) can be efficiently computed by standard gradient-based optimizers. In Sec. \ref{sec:efficient_approximation}, we will propose an algorithm which can efficiently approximate $\text{OFU}^R$ with a few steps of gradient descent. 
 
 An important question is whether selecting action with maximum $\text{OFU}^R$ achieves good EE trade-off, or whether $\text{OFU}^R(x_t,a)$ is a reasonable upper confidence bound over $f_{\theta_{t-1}}(x_t,a)$. 
 To provide more insights, consider the case where $\mathcal{L}=\mathcal{R}$. It is easy to see that $\forall (x,a), f_{\hat{\theta}(x,a)}(x,a)\geq f_{\theta_{t-1}}(x,a)$. Thus, given $g$ is monotonically increasing and $g(0)=0$, we have $\text{OFU}^R(x,a)$ is an upper bound of $f_{\theta_{t-1}}(x,a)$.  Intuitively, if we can find a parameter $\hat{\theta}$ that increases the value of $f_{\theta}(x,a)$ without significantly increasing  $\mathcal{R}(\theta;D)$, then the uncertainty on the reward of $(x,a)$ would be large. Consequently, a large bonus would be added to such an action by Eq.~(\ref{eq:rofu}).  Moreover, $g$ and $\eta$ cooperatively control the scale of the confidence bonus:  the difference between $f_{\hat{\theta}(x,a)}(x,a)$ and $f_{\theta_{t-1}}(x,a)$ 
increases as $\eta$ decreases and $g$ controls the contribution of such difference to the final bonus.  
In Section~\ref{subsec:regret-analysis}, we show that with proper $g,\eta,\mathcal{L}$ and $\mathcal{R}$  we can achieve near-optimal regret bound on contextual bandit when the reward function is simple. 
\begin{remark}
 It is noteworthy that the choices of $\mathcal{R},\mathcal{L},\eta$ and $g$ are crucial to the performance on EE trade-off. For example, if we select $\mathcal{R}:=\mathbb{P}(D|\theta)$ which is used in Set-OFU, then $f_{\hat{\theta}(x,a)}=+\infty,\forall (x,a)$. Then our method would  uniformly sample actions. 
 
\end{remark}

 \subsection{Regret analysis} \label{subsec:regret-analysis}
 
 The regret bounds for ROFU are developed by drawing connections between ROFU with existing near-optimal algorithms.  Then the regret of ROFU in these cases is also near-optimal. 
 More specifically, we show that for MAB, LCB and KCB, the exact solution of Eq. (\ref{eq:rofu}) is equivalent to UCB1, LinUCB and KernelUCB respectively.   For  NTKCB, $\text{OFU}^R$ has similar properties to the upper confidence bound in NeuralUCB. We summarize these results in Table \ref{table:theory}. 
  
  \begin{table*}[ht]
\begin{center}
	\begin{tabular}{c  c  c  c }
	    \toprule
      &  MAB &  Linear/Kernel-CB & NTK CB\\
     	\midrule	
     	$\mathcal{R}$ & $|D|\text{MSE}(\theta;D)$  & $\|\theta\|^2 + |D|\text{MSE}(\theta;D)$ & See Eq. (\ref{eq:ntk-regularizer})\\
     	
     	$\mathcal{L}$& $\text{MSE}(\theta;D)$ &$\text{MSE}(\theta;D)$&See Eq. (\ref{eq:ntk-loss})\\ 

     	$\eta$ & $\frac{1}{16\log |D|}$ & $1/2$ & $1/(2\gamma^2)$\\
     	
     	$g(w)$& $\sqrt{w}$ & $\sqrt{w}$ & $\sqrt{w}$\\
     	
     	Equivalent/Related to & UCB1 & Lin/KernelUCB & NeuralUCB\\
     	Regret & $O(\log T)$ & $\tilde{O}(\sqrt{T})$ & $\tilde{O}(\sqrt{T})$\\
     	\bottomrule
	\end{tabular}
\end{center}
\caption{Theoretical results for ROFU on contextual bandit tasks with simple reward functions. 
}\label{table:theory}
\end{table*}
  
\textbf{Equivalence to UCB1, LinUCB and KernelUCB}
  
  Specifically, we  use MSE-variants for both $\mathcal{L}$ and $\mathcal{R}$ and $g(w)=\sqrt{w}$, henceforth. After fixing  $\mathcal{L}$, $\mathcal{R}$ and $g$, we can tune $\eta$ to make the solution of Eq. (\ref{eq:rofu}) equal to the upper confidence bound in UCB1, LinUCB or KernelUCB.  Recall $\text{MSE}(\theta;D):=\frac{1}{|D|}\sum_{t\leq |D|}(f_{\theta}(x_t,a_t)-r_t)^2$.  The following theorem shows the equivalence between ROFU and LinUCB. 
 
 \begin{Thm}\label{thm:lcb}
  With $\mathcal{L}(\theta;D):=\text{MSE}(\theta;D)$, $\mathcal{R}(\theta;D):=\|\theta\|^2+|D|\text{MSE}(\theta;D)$ , $
     \eta(x,a,D):= 1/2$ and $g(w)=\sqrt{w}$.
   We have $OFU^{R}$ is identical to the upper confidence bound used in LinUCB.
 \begin{proof}
 Recall that in LCB, $f_{\theta}(x,a):=\phi(x,a)^\top \theta$. By setting the derivative of Eq. (\ref{eq:rofu_optim}) as zero, we have $
     \phi(x,a)-2\eta(x,a,D)((\mI+\sum_{(x_i,a_i,r_i)\in D}\phi(x_i,a_i)\phi(x_i,a_i)^\top)\theta-\sum_{(x_i,a_i,r_i)\in D}\phi(x_i,a_i)r_i)=0$. 
 With elementary algebra, we can see that $
     \hat{\theta}(x,a)=\theta_{t-1}+\frac{1}{2\eta(x,a,D)}(I+\sum_{(x_i,a_i,r_i)\in D}\phi(x_i,a_i)\phi(x_i,a_i)^\top)^{-1}\phi(x,a)$. 
     
 Let $\mZ:=(\mI+\sum_{(x_i,a_i,r_i)\in D} \phi(x_i,a_i)\phi(x_i,a_i)^\top)$. Given $g(w)=w^{1/2}$, inserting $\hat{\theta}(x,a)$ and $\eta$ into Eq. (\ref{eq:rofu_optim}), we have $
     \text{OFU}^R(x,a) =\phi(x,a)^\top \theta_{t-1} + (\phi(x,a)^\top \mZ^{-1}\phi(x,a))^{\frac{1}{2}}$  which is used as the upper confidence bound in LinUCB. We finish the proof. 
 \end{proof}
 
 \end{Thm}

 The derivation in Thm \ref{thm:lcb} can be directly used to show that $\text{OFU}^{R}$ is identical to upper confidence bound used in  KernelUCB as well. And  the equivalence between ROFU and UCB1 can be proved by similar techniques. We postpone the derivation into Appendix. 

\begin{remark}
From the derivation in Thm \ref{thm:lcb}, it is easy to see that when $g(w)=w^b, b>0$, we can also tune $\eta$ to make ROFU identical to LinUCB. However, for $b\neq 1/2$, $\eta$ is a complex function depending on $x,a$ and $D$. Thus, we only consider $g(w)=w^{1/2}$. 
\end{remark}


\textbf{Connection to NeuralUCB and sublinear regret for NTK contextual bandit.}

The NeuralUCB \cite{zhou2020neural} uses a neural network to learn the reward function and constructs a UCB by using the DNN-based random feature mappings. For learning the reward function, NeuralUCB exploits the recent progress of optimizing a DNN in the Neural Tangent Kernel (NTK) regime \cite{allen2018convergence,du2019gradient,zou2019improved}, which shows that gradient descent can find global minima of  a loss for a finite set of training samples if the network is sufficiently wide and the learning rate is sufficiently small.  

Intuitively,  NeuralUCB takes the gradient $\nabla_\theta f_\theta(x_t,a)$ as a random feature to represent  the observation $(x_t,a)$. Then, it constructs a UCB as in a linear bandit, i.e., $\text{NeuralUCB}(x_t,a) = f_{\theta_{t-1}}(x_t,a) + \gamma \sqrt{\nabla_\theta f_{\theta_{t-1}}(x_t,a)^\top \tilde{\mZ}_{t-1}^{-1}\nabla_\theta f_{\theta_{t-1}}(x_t,a)/m }$, where $\gamma$ is some predefined constant, $m$ is the network width and $\tilde{\mZ}_{t-1}$ is a matrix with the contextual information. Specifically, NeuralUCB achieves $\tilde{O}(\sqrt{T})$ regret. 

We next show that ROFU can also achieve a similar regret in the NTK regime by carefully choosing the regularizer $\gR$. To hide the technical exposure of NTK regime, we adopt the same assumptions and the same model setup as those in \cite{zhou2020neural}. We  bound the difference between ROFU and NeuralUCB in the NTK regime, and then show that ROFU can achieve a similar regret bound to NeuralUCB. 

At the start, we assume that the network takes the following form
\begin{flalign}
f_\theta(\vx) = \sqrt{m} \mW_L \sigma(\mW_{L-1}\sigma(\cdots\sigma(\mW_1 \vx))),
\end{flalign}
where $\vx$ is composed of $(x,a)$, the $\sigma(\cdot)$ is the point-wise activation function, $L$ is the network depth and $m$ is the network width. Furthermore, $\mW_1 \in \sR^{m\times d}, \mW_l \in \sR^{m\times m}$ for $l\in \{2, ..., L-1\}, \mW_L \in \sR^{1\times m}$ and $\theta$ collects all these matrices into a long vector with dimension $p$.

To see the connection with NeuralUCB, in Algorithm \ref{alg:generalROFU} we choose the loss as
\begin{flalign}
\gL(\theta; D) := \sum_{(x,a,r)\in D} (f_\theta(x,a)-r)^2 +m\lambda \|\theta- \theta_0\|^2, \label{eq:ntk-loss}
\end{flalign}
where $\theta_0$ is the network parameter at initialization, $m$ is the network width. Here, the MSE loss is penalized by the distance between $\theta$ and its initialization $\theta_0$ because the analysis technique for NTK bandit requires that the network parameters stay within a neighborhood around the initialization \cite{zhou2020neural}. We note that the penalizing coefficient written as $m\lambda$ is for proof convenience and the comparison with NeuralUCB. We further choose the regularizer as
\begin{flalign}
\gR(\theta; D_{t-1}) := \frac{1}{\eta}(f_{\theta}(x_t,a) - \tilde{f}_{\theta}(x_t,a;\theta_{t-1})) + \tilde{\gR}(\theta, D_{t-1};\theta_{t-1}), \label{eq:ntk-regularizer}
\end{flalign}
where $\theta_{t-1}$ is the last iterate that minimizes the loss on $D_{t-1}$. Moreover, $\tilde{f}_{\theta}(x_t,a;\theta_{t-1})=f_{\theta_{t-1}}(x_t,a) + \langle \nabla_\theta f_{\theta_{t-1}}(x_t,a),\theta -\theta_{t-1}\rangle$ and $\tilde{\gR}(\theta, D_{t-1};\theta_{t-1}) :=\sum_{(x,a,r)\in D_{t-1}} (f_{\theta_{t-1}}(x,a) + \langle \nabla_\theta f_{\theta_{t-1}}(x,a),\theta -\theta_{t-1}\rangle -r)^2 + m\lambda \|\theta- \theta_0\|^2$ are the first order Taylor expansions of $f_{\theta}(x_t,a)$ and $\gL(\theta; D_{t-1})$  at $\theta_{t-1}$, respectively. Then equivalently, $\hat{\theta}(x_t,a)$  is calculated by 
\begin{flalign}
\hat{\theta}(x_{t},a) := \argmax \tilde{f}_{\theta}(x_t,a;\theta_{t-1}) - \eta \tilde{\gR}(\theta, D_{t-1};\theta_{t-1}), \label{eq:ntk-theta-hat}
\end{flalign}
 where  $\tilde{f}_{\theta}(x_t,a;\theta_{t-1})$ and $\tilde{\gR}(\theta, D_{t-1};\theta_{t-1})$ are defined as above.

We assume that the network is trained in the NTK regime  with learning rate $\beta$ (see details in Algorithm \ref{alg:ntktrainNN} of Appendix \ref{app:sec:ntkregret}). Then  if choosing $g(w) = \sqrt{w}$ in Eq. (\ref{eq:rofu}) and $\eta = 1/(2\gamma^2)$ where $\gamma$ is a factor same as the one used in NeuralUCB, we have a regret bound for ROFU  as follows.

\begin{restatable}{Thm}{ntkregret}\label{thm:ntk}

For convenience, let $L$ denote the network depth, $T$ denote the time horizon, $A$ denote the action set  and $\lambda_0$ denote the minimal eigenvalue of the neural tangent kernel matrix at initialization. 
If the loss and regularizer are chosen as Eq. (\ref{eq:ntk-loss}) and Eq. (\ref{eq:ntk-regularizer}) and the network width satisfies $m\ge \text{poly}(T,|A|, L, \lambda^{-1}, \lambda_0^{-1})$ and the learning rate $\beta \le C(mTL+m\lambda)^{-1}$, then with high probability, the regret of ROFU satisfies that $R_T \le \tilde{O}(\sqrt{T})$ where $\tilde{O}(\cdot)$ hides the $\log T$ terms, the dependence on the problem dimension and other factors that are irrelevant with $T$. 
\end{restatable}
 
The  idea of proof is to connect the formula of ROFU with the NeuralUCB and then bound their difference such that the order of the  regret bound is not affected. We leave the full proof into Appendix \ref{app:sec:ntkregret}. This justifies that the ROFU can also achieve a near optimal regret bound in the NTK regime. Nonetheless, in sharp contrast with NeuralUCB where computing the inverse of a large matrix $\mZ_{t-1}$ is necessary in each step, ROFU admits a very efficient implementation as shown below.

\subsection{An efficient algorithm to compute $\text{OFU}^R$}\label{sec:efficient_approximation}

In order to apply ROFU, we have to efficiently  
solve or approximate Eq.~(\ref{eq:rofu_optim}). In cases of MAB, KCB, we can calculate closed-form solutions for Eq.~(\ref{eq:rofu_optim}). When $f_{\theta}$ is a deep neural network, it is time-consuming to exactly solve Eq.~(\ref{eq:rofu_optim}). Naturally, one can use gradient descent methods to approximately solve Eq.~(\ref{eq:rofu_optim}). However, it is still not manageable to optimize from scratch for every $(x_t,a)$. We propose that the optimization starts from $\theta_{t-1}$ for $(x_t,a)$. 

More specifically, when $f_{\theta}$ is a neural network and $\mathcal{L}$ is differentiable,  we can optimize the parameters $\theta$  with gradient ascent. We approximately solve Eq.~(\ref{eq:rofu_optim}) by executing a few steps of gradient ascent starting from $\theta_{t-1}$. That is, $\hat{\theta}_{j+1}=\hat{\theta}_{j}+\kappa \nabla_{\theta} (f_{\hat{\theta}_j}(x_t,a)-\eta \mathcal{R}(\hat{\theta}_j;D))$ with $\hat{\theta}_0:=\theta_{t-1}$, where $\kappa$ is the step size. The above implementation is summarized in Alg. \ref{alg:acc_rofu}. 

Alg. \ref{alg:acc_rofu} essentially performs a local search around $\theta_{t-1}$. Indeed it brings extra benefit for the optimization by starting from $\theta_{t-1}$ in this case. This is because intuitively  $\theta_{t-1}$  often gets closer to $\hat{\theta}(\cdot,a)$  when $t$ is larger. For example, in MAB, $\theta_{t-1}(\cdot,a)=\Bar{\mu}_{t-1}(a)$ where $\Bar{\mu}_{t-1}(a)$ is the empirical mean of action $a$ and $\hat{\theta}_{t-1}(\cdot,a)=\Bar{\mu}_{t-1}(a)+1/n(a)$ where $n(a)$ is the number of pulls on $a$. Thus, $\theta_t$ is close to $\hat{\theta}$ when $n(a)$ is large. 

It is easy to see that the time complexity of Alg. \ref{alg:acc_rofu} is $O(TM|A|p)$. In Sec. \ref{sec:experiment}, we can see that the regret of ROFU is low in practice with very small $M$. 

\begin{remark}\label{remark:convexity}
 For MAB, linear/kernel bandit and NTK bandit,  $-(f_{\theta}(x_t,a)-\mathcal{R}(\theta;D))$ is strongly convex and smooth. Assuming that it is $\alpha$ strongly convex and $\zeta$ smooth,  by Thm 2.1 in \cite{needell2014stochastic}, we have $\mathbb{E}\|\hat{\theta}_M(x_t,a)-\hat{\theta}(x_t,a)\|^2\leq (1-2\kappa\alpha(1-\kappa\zeta))^M\|\theta_{t-1} - \hat{\theta}(x_t,a)\|^2 + \frac{\kappa\nu^2}{\alpha(1-\kappa\zeta)}$, where $\mathbb{E}$ is over the randomness of SGD, $\kappa$ is the learning rate and  $\nu^2$ is the expected gradient norm square at  $\hat{\theta}(x_t,a)$.
\end{remark}

We note that the convex and smooth property for NTK bandit is due to the choice of $\gR(\theta;D)$ in Eq.~(\ref{eq:ntk-regularizer}). 
Based on Remark \ref{remark:convexity}, we can expect Alg. \ref{alg:acc_rofu} converges fast for these cases, especially, which can leverage the property that $\|\theta_{t-1} - \hat{\theta}(x_t,a)\|$ gets smaller as $t$ gets larger.  When $f_{\theta}$ is general and complex, there is no theoretical guarantee on the convergence rate, but in practice, stochastic gradient descent with a good initialization usually  converges fast. We note that NeuralUCB can have a fast implementation by approximating $\mZ$ with its diagonal elements, which, however, can significantly compromise the regret in practice. Please see Sec. \ref{sec:real-world} for empirical evaluations. 

\begin{algorithm}[hbt!]
\caption{An efficient implementation to estimate $\text{OFU}^R$}
\label{alg:acc_rofu}  
\begin{algorithmic}[1]

\STATE \textbf{Input}:   Dataset $D_{t-1}$,  $\theta_{t-1}=\argmin_{\theta}\mathcal{L}(\theta;D_{t-1})$ and context-action pair $x_t, a$. Learning rate $\kappa$ and training steps $M$.
    \STATE Observe context $x_t$.
        \STATE Set $\hat{\theta}_0(x_t,a):=\theta_{t-1}$.
        \FOR{$j=1,..., M$}
            \STATE \label{line:gradient} $\hat{\theta}_j(x_t, a):=\hat{\theta}_{j-1}(x_t,a)+\kappa \tilde{\nabla}(f_{\hat{\theta}_{j-1}(x_t, a)}(x_t,a) - \eta\mathcal{R}(\hat{\theta}_{j-1}(x_t,a);D_{t-1}) )$, where $\tilde{\nabla}$ is an estimator of  gradient. 
            
        \ENDFOR
    \RETURN $\text{OFU}^R(x_t,a):=f_{\theta_{t-1}}(x_t,a)+\sqrt{\max(0,f_{\hat{\theta}_M(x_t,a)}(x_t,a)-f_{\theta_{t-1}}(x_t,a))}$.
\end{algorithmic}
\end{algorithm}

%% file: contents/related_work.tex
\section{Related work}

The contextual bandit problem has been extensively studied and applied in machine learning. Besides OFU, the most used and studied contextual bandit algorithms can be mainly divided into two categories: $\epsilon$-greedy \cite{sutton2018reinforcement} and Thompson sampling \cite{thompson1933likelihood}.

$\epsilon$-greedy greedily selects the action $a=\argmax_{a'}f_{\theta_{t-1}}(x, a')$ with probability $1-\epsilon$ and randomly selects an action to explore with probability $\epsilon$. While being sub-optimal on EE trade-off, $\epsilon$-greedy is simple, general and computationally efficient. Thus, $\epsilon$-greedy is the most widely used algorithm in complex tasks. 

The basic idea of Thompson sampling \cite{thompson1933likelihood} is to maintain a posterior over parameters and make decisions according to the samples from the posterior. However, Thompson sampling can be applied only if we can access the posterior and use approximate inference methods which can be computationally inefficient and may significantly hurt the performance of regret \cite{deep_bayesian_bandit_showdown,phan2019thompson}.

There are several works extending OFU and Thompson sampling to the case of non-linear contextual bandits. NeuralUCB \cite{zhou2020neural} and Neural-TS \cite{zhang2021neural} extend OFU and TS to NTK bandits and they provide regret bounds in the order of $\tilde{O}(\sqrt{T})$. Our analysis for NTK bandit is largely inspired by their work. However, as mentioned in Sec. \ref{sec:preliminary}, both NeuralUCB and NeuralTS need to calculate the inversion of a matrix with dimension equal to the number of parameters in DNN, which is computationally inefficient for large DNN models. Moreover, the neural-linear model is studied in \cite{zahavy2019deep,deep_bayesian_bandit_showdown}, where they take the output of the second last layer of a DNN as feature map. These algorithms sometimes work well, but have no theoretical guarantee on regret.

%% file: contents/experiment.tex
\section{Experiment}\label{sec:experiment}

We now empirically evaluate ROFU. For simplicity, in all our experiments, we set $\eta=1, g(w)=\sqrt{w}$, $\mathcal{L}(\theta;D)=\text{MSE}(\theta;D)$ and $\mathcal{R}(\theta;D)=|D|\text{MSE}(\theta;D)$. That is, $\theta_{t-1}$ is trained to minimize $\text{MSE}(\theta;D)$ with standard optimizer \footnote{We train $\theta_{t-1}$ with stochastic gradient descent starting from $\theta_{t-2}$ in all the experiments.} and $\hat{\theta}(x,a)$ is trained to maximize $f_{\theta}(x,a)- |D| \text{MSE}(\theta;D)$ using Alg. \ref{alg:acc_rofu}.

 \begin{figure*} [t!]
\subfigure{
\includegraphics[width=1.\textwidth]{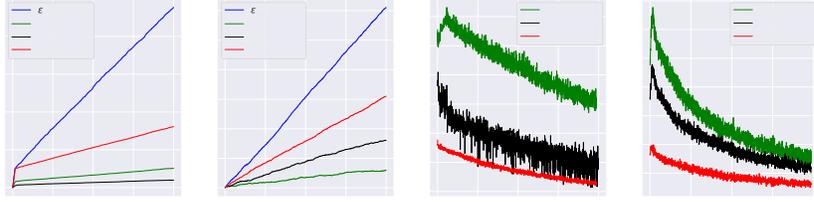}  
}
\caption{Ablation study on MLP and ResNet bandits. Notation: $M=1/5/10$ means that we run $1/5/10$ gradient descent updates in Alg. \ref{alg:acc_rofu}.}
 \label{fig:ablation}
\end{figure*}

\subsection{Analysis on MLP and ResNet bandits}\label{sec:abl_study}
As discussed in Sec. \ref{sec:introduction}, a contextual bandit algorithm should be efficient in trading off EE when reward is generated from a complex function  while keeping a low cost on computational resources. From Alg. \ref{alg:acc_rofu}, we can see that the time complexity of ROFU is determined by the training step $M$. Thus, we here evaluate the regret of ROFU when $M$ is small. 


To evaluate the performance of ROFU in Alg. \ref{alg:acc_rofu} on complex tasks, we consider two contextual bandits with a DNN as the simulator. That is, $r(x_t, a)$  is generated by a DNN model. We consider two popular  DNN architectures to generate rewards: 2-layer MLP and 20-layer ResNet with CNN blocks and Batch Normalization as in \cite{he2016deep}.  
We summarize other information of the two tested bandits in Table. \ref{table:info}.

\begin{table}[ht!]\label{table:nonlinear-cb-results-table}
\begin{center}
\begin{adjustbox}{width=\textwidth} 
	\begin{tabular}{c c  c  c  c c  c }
	    \Xhline{2\arrayrulewidth}
     	Bandit	 & \small Layer & \small Context Dim  & \small \# Arms & \small NN Parameters & Context Distribution & Noise  \\
     	\hline
     	\small MLP & $2$ & $10$  & $10$ & Random & Gaussian & $\mathcal{N}(0, 0.05)$  \\
     	\small ResNet & $20$ & $3\times 32\times 32$ & $10$ & Train on Cifar10 & Uniform & $\mathcal{N}(0, 0.5)$ \\
        \hline
	\end{tabular}
\end{adjustbox}
\end{center}
\caption{Basic information about MLP and ResNet bandits.}
\label{table:info}
\end{table}

We use  DNNs with larger size for training in Alg. \ref{alg:acc_rofu}. More specifically, for MLP-bandit,   $f_\theta$ is chosen as a 3-layer MLP and for ResNet20-bandit, $f_\theta$ is chosen as ResNet32. Each experiment is repeated for 16 times.  We present the regret and confidence bonus in Fig. \ref{fig:ablation}. From Fig. \ref{fig:ablation}, we can see that (1) ROFU can achieve a  small regret on both tasks with a considerably small $M$ even for very large DNN model; (2) The confidence bonus monotonically increases with $M$ . For each $M$, the confidence bonus converges to $0$ as expected. Moreover, while the regret seems sensitive to the value of $M$, the regrets of ROFU with $M=5,10$ are much smaller than the case of $M=1$ and $\epsilon$-greedy.

 \begin{figure*} [t!]
\centering
\subfigure{\label{fig:covertype}\includegraphics[width=1.\textwidth]{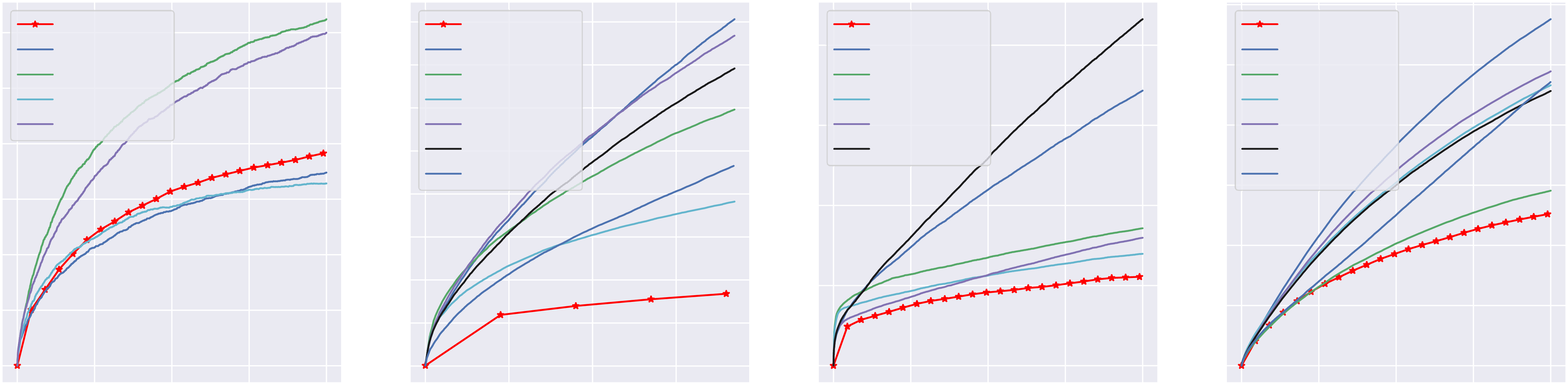}
}
 \quad
 \hspace*{-.0in}
 \subfigure{\label{fig:covertype}\includegraphics[width=.75\textwidth]{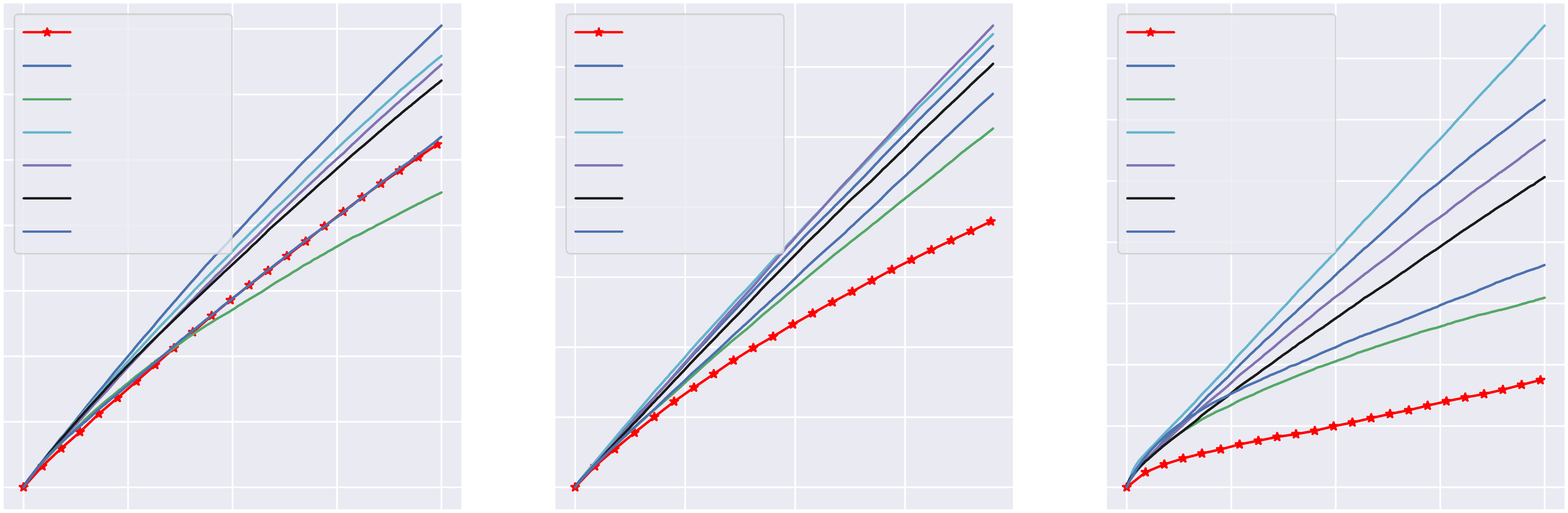}
}

\caption{Evaluations on non-linear contextual bandits.  }
 \label{fig:ncb}
\end{figure*}

\subsection{Performance comparison on real-world datasets}\label{sec:real-world}

To evaluate ROFU against powerful baselines, we conduct experiments on  contextual bandits which are created from real-world datasets, following the setting in \cite{deep_bayesian_bandit_showdown}. For example, suppose that   $D:=\{(x_t, c_t)\}_{t\leq T}$ is a $K$-classification dataset where $x_t$ is the feature vector and $c_t\in [K]$ is the label. We create a  contextual bandit problem as follows: at time step $t\leq T$, the agent observes context $x_t$, and then  takes an action $a_t$. The agent receives high reward if it successfully predicts the label of $x_t$. For non-classification dataset, we can turn it into contextual bandit in similar ways. For the details of these bandits, please refer to \cite{deep_bayesian_bandit_showdown}. 

For baselines, we consider NeuralUCB \cite{zhou2020neural} and Thompson sampling variants from \cite{deep_bayesian_bandit_showdown}. It is noteworthy that we only evaluate the algorithms in \cite{deep_bayesian_bandit_showdown} with relatively small regrets. We directly run the code provided by the authors. For ROFU, we fix $M=5$ for all experiments. For other hyper-parameters, please see Appendix \ref{app:detail}. We repeat each algorithm for $16$ times on each dataset. 

We report the regret $R_T = \mathbb{E}\sum_{t\leq T}r(x_t,a'_t) - \mathbb{E}\sum_{t\leq T}r(x_t, a_t)$ where $a'_t=\argmax_{a\in A}f_{\theta'}(x_t,a)$ and $\theta'$ is the parameter trained on the dataset. For hyperparameters tuning and more details, please see Appendix \ref{app:detail}. The results are presented in Fig. \ref{fig:ncb} and Table \ref{table:final_regrets}. 
We found that the regret of NeuralUCB is occasionally linear. This might be because that NeuralUCB uses a diagonal matrix to approximate $Z$ to accelerate. Moreover, we can see that ROFU significantly outperforms these baselines in terms of regret. 

\begin{table}[H]\label{table:nonlinear-cb-results-table}
\begin{center}
\begin{adjustbox}{width=1.\textwidth} 
	\begin{tabular}{c c  c  c  c c  c  c  c}
	    \Xhline{2\arrayrulewidth}
     		 & \small Mean & \small Census  & \small Jester & \small Adult & \small Covertype & \small Statlog &\small Financial &\small Mushroom \\
     	\hline
     	\small Dropout & $1.75\scriptstyle\pm 0.80$ & $1.51\scriptstyle\pm  0.10$  & $1.34\scriptstyle\pm 0.14$ & $\mathbf{1.00\scriptstyle\pm 0.09}$ & $1.14\scriptstyle\pm 0.13$ & $1.54\scriptstyle\pm 0.87$ & $3.50\scriptstyle\pm 0.60$ & $2.21\scriptstyle\pm0.42$ \\
     	\small Bootstrap & $2.23\scriptstyle\pm 1.00$ & $2.51\scriptstyle\pm 0.16$ & $1.72\scriptstyle\pm 0.11$ & $1.43\scriptstyle\pm 0.10$ & $1.93\scriptstyle\pm 0.13$ & $1.43\scriptstyle\pm 1.57$ & $4.52\scriptstyle\pm 2.29$ & $2.04\scriptstyle\pm 0.48$\\
     	\small ParamNoise & $2.30\scriptstyle\pm 1.12$ & $2.28\scriptstyle\pm 0.23$ & $1.59\scriptstyle\pm 0.14$ & $1.37\scriptstyle\pm 0.10$ & $1.80\scriptstyle\pm 0.20$ & $3.88\scriptstyle\pm 6.40$ & $4.07\scriptstyle\pm 1.76$ & $1.06\scriptstyle\pm 0.32$ \\
     	\small NeuralLinear & $1.82\scriptstyle\pm 0.69$ & $3.24\scriptstyle\pm 0.47$ & $1.70\scriptstyle\pm 0.13$ & $1.46 \scriptstyle\pm 0.12$ & $1.84\scriptstyle\pm 0.19$ & $1.25\scriptstyle\pm 0.11$ & $2.25\scriptstyle\pm 0.35$ & $\mathbf{1.00\scriptstyle\pm 0.38}$  \\
     	\small Greedy & $2.47\scriptstyle\pm1.12$ & $2.76\scriptstyle\pm 1.24$ & $1.65\scriptstyle\pm 0.10$ & $1.56\scriptstyle\pm 0.11$ & $2.27\scriptstyle\pm 0.23$ & $3.08\scriptstyle\pm 4.91$ & $4.74\scriptstyle\pm 2.31$ & $1.20\scriptstyle\pm 0.41$ \\
     	\small NeuralUCB & $9.76\scriptstyle\pm 14.02$ & $1.72\scriptstyle\pm 0.12$ & $1.47\scriptstyle\pm 0.08$ & $1.18\scriptstyle\pm 0.05$ & $1.86\scriptstyle\pm 0.16$  & $41.42\scriptstyle\pm 69.51$ & $2.74\scriptstyle\pm 0.50$ & $17.29\scriptstyle\pm 7.45$ \\
     	\small ROFU (ours) & $\mathbf{1.05\scriptstyle\pm 0.09}$ & $\mathbf{1.00\scriptstyle\pm 0.09}$ & $\mathbf{1.00\scriptstyle\pm 0.20}$ & $1.17\scriptstyle\pm 0.06$ & $\mathbf{1.00 \scriptstyle\pm 0.14}$ & $\mathbf{1.00\scriptstyle\pm 0.24}$ & $\mathbf{1.00\scriptstyle\pm 0.60}$ & $1.22\scriptstyle\pm 0.37$\\
        \hline
	\end{tabular}
\end{adjustbox}
\end{center}
\caption{The final regret of each algorithm. The regrets are normalized according to the algorithm with smallest regret.}
\label{table:final_regrets}
\end{table}


%% file: contents/conslusion.tex
\section*{Conclusion and future work}

In this work, we propose an OFU variant, called ROFU, which can be applied to non-linear contextual bandit. We show that the regret of ROFU is (near-)optimal on various contextual bandit models. Moreover, we propose an efficient algorithm to approximately compute the upper confidence bound. Thus, ROFU is efficient in both computation and EE trade-off, which are empirically verified by our experimental results.

EE trade-off is a fundamental problem that lies in the heart of sequential decision making.   However, the huge computational cost of (near-)optimal EE trade-off algorithms significantly limits the application, especially on complex domain. We hope our method could inspire more algorithms to efficiently trade-off EE for sequential decision-making tasks beyond contextual bandit, such as deep reinforcement learning.

%% file: appendix/ucb_proof.tex
\section{Analysis on multi-armed bandit}
The following theorem presents the equivalence between ROFU and UCB1 \cite{auer2002finite}. 

\begin{Thm}\label{thm:ucb1}
Let $n_a$ denote the number of pulls on action $a$ and $\Bar{\theta}_a$ denote the empirical mean of rewards on action $a$. With $\mathcal{L}(\theta;D)=\text{MSE}(\theta;D)$, $\mathcal{R}(\theta;D)=|D|\text{MSE}(\theta;D)$, $\eta(a)=\frac{1}{2n_a(8\log |D|/n_a)^{1/2b}}$ and $g(w)=w^b$, we have $OFU^R(a)=\Bar{\theta}_a + \sqrt{8\log t/n_a}$ which is used in UCB1.
\begin{proof}
Let the derivative of Eq. (\ref{eq:rofu_optim}) equal zero, we have $\hat{\theta}(a)=8\log t/n_a + \Bar{\theta}_a$. Inserting to Eq. (\eqref{eq:rofu}), we finish the proof. 
\end{proof}
\end{Thm}

%% file: contents/ntkbandit.tex
\section{Analysis on neural tangent kernel bandit} \label{app:sec:ntkregret}

\ntkregret*

The \textbf{step 6} in Algorithm \ref{alg:generalROFU} is accomplished by running gradient descent for $J$ iterations with learning rate $\beta$ to minimize the loss, as given by Algorithm \ref{alg:ntktrainNN}. In the following analysis, we assume that $\theta^{(J)}$ satisfies the first order optimality of minimizing $\mathcal{L}(\theta; D)$ for the ease of the derivation and $\theta_0$ is the randomly initialized starting point.



    


\begin{algorithm}[htb!]
\caption{TrainNN($D, \beta, \theta_0$)}
\label{alg:ntktrainNN}  
\begin{algorithmic}[1]
\STATE Define $\mathcal{L}(\theta; D) = \sum_{(x,a,r)\in D} (f(x,a; \theta)-r)^2 +m\lambda \|\theta-\theta_0\|^2$.
\STATE Let $\theta^{(0)} \leftarrow \theta_0$.
\FOR{j=0, ..., J-1}

    \STATE $\theta^{(j+1)} = \theta^{(j)} - \beta \nabla \mathcal{L}(\theta^{(j)})$.
\ENDFOR
\RETURN $\theta^{(J)}$. 

\end{algorithmic}
\end{algorithm}

By choosing the regularizer $\gR(\theta; D_{t-1})$ as Eq.~(\ref{eq:ntk-regularizer}), the procedure of searching $\hat{\theta}(x_t,a)$ (Eq.~(\ref{eq:rofu_optim})) is equivalent to
\begin{flalign}
\hat{\theta}(x_{t},a) = \argmax \tilde{f}_{\theta}(x_t,a;\theta_{t-1}) - \eta \tilde{\gR}(\theta, D_{t-1};\theta_{t-1}),  \label{eq:ntk-hat-theta}
\end{flalign}
where 
\begin{flalign*}
&\tilde{f}_{\theta}(x_t,a;\theta_{t-1})=f_{\theta_{t-1}}(x_t,a) + \langle \nabla_\theta f_{\theta_{t-1}}(x_t,a),\theta -\theta_{t-1}\rangle,\\
&\tilde{\gR}(\theta, D_{t-1};\theta_{t-1}) :=\sum_{(x,a,r)\in D_{t-1}} (f_{\theta_{t-1}}(x,a) + \langle \nabla_\theta f_{\theta_{t-1}}(x,a),\theta -\theta_{t-1}\rangle -r)^2 + m\lambda \|\theta- \theta_0\|^2 
\end{flalign*} 
are the first order Taylor expansions of $f_{\theta}(x_t,a)$ and $\gL(\theta; D_{t-1})$  at $\theta_{t-1}$, respectively. We note that the special choice of regularizer $\gR(\theta;D_{t-1})$ Eq.~(\ref{eq:ntk-regularizer}) is to derive a simple form of $\hat{\theta}(x_t,a)$.
 
In the following analysis, we first connect ROFU with NeuralUCB via an informal argument. Then we bound the approximations in the informal argument and show they will not affect to establish the sublinear regret in the order sense. We use notation  $h_{\theta_{t-1}}(x_t,a)$ to denote $\nabla_\theta f_{\theta}(x_t,a)\large|_{\theta = \theta_{t-1}}$.

\textbf{An intuitive procedure.}

With the above choices, we can see that the procedure Eq. (\ref{eq:ntk-hat-theta}) is equivalent to
\begin{flalign}
\hat{\theta}(x_{t},a) = \argmax_{\theta}  \langle h_{\theta_{t-1}}(x_t,a),\theta -\theta_{t-1}\rangle + \eta \|\theta - \theta_{t-1}\|_{m\mZ_{t-1}}^2, \label{eq:ntk-hat-theta2}
\end{flalign}
where 
\begin{flalign}
\mZ_{t-1} := \lambda \mI + \frac{1}{m}\sum_{(x,a,r)\in D_{t-1}} h_{\theta_{t-1}}(x,a)h_{\theta_{t-1}}(x,a)^\top \label{eq:Zt-1}
\end{flalign} by assuming that $\theta_{t-1}$ satisfies the first-order optimality condition of minimizing the $\gL(\theta;D_{t-1})$.

Solving the problem Eq. (\ref{eq:ntk-hat-theta2}), we obtain $$\hat{\theta}(x_{t},a)= \theta_{t-1} +\frac{1}{2\eta m}\mZ_{t-1}^{-1} h_{\theta_{t-1}}(x_t,a). $$

Hence if the function $g(\cdot)$ in Eq. (\ref{eq:rofu}) is given by $g(x) = \sqrt{x}$, then the $\text{OFU}^{R}$ becomes
\begin{flalign}
\text{OFU}^{R}(x_t,a) &= f_{\theta_{t-1}} (x_t,a) + \sqrt{f_{\hat{\theta}(x_t,a)} (x_t,a)-f_{\theta_{t-1}} (x_t,a)} \nonumber\\
& \approx  f_{\theta_{t-1}}(x_t,a) +  \frac{1}{\sqrt{2\eta}}\sqrt{h_{\theta_{t-1}}(x_t,a)^\top \mZ_{t-1}^{-1} h_{\theta_{t-1}}(x_t,a)/m}, \label{eq:ntk-rofu}
\end{flalign}
where the approximation is because of using the first order Taylor approximation of $f_{\hat{\theta}(x_t,a)} (x_t,a)$ at $\theta_{t-1}$ to replace $f_{\hat{\theta}(x_t,a)} (x_t,a)$.

Similar to the formula of (\ref{eq:ntk-rofu}), one uses 
\begin{flalign}
U_{t,a} = f_{\theta_{t-1}}(x_t,a) +  \gamma_{t-1}\sqrt{h_{\theta_{t-1}}(x_t,a)^\top \tilde{\mZ}_{t-1}^{-1} h_{\theta_{t-1}}(x_t,a)/m} \label{eq:neuralubc}
\end{flalign} 
as the upper confidence bound in NeuralUCB  \cite[Algorithm 1]{zhou2020neural}. One difference is that they \cite{zhou2020neural} use  \begin{flalign}
\tilde{\mZ}_{t-1} := \lambda \mI + \frac{1}{m}\sum_{\tau=1}^{t-1} h_{\theta_{\tau-1}}(x_\tau,a_\tau) h_{\theta_{\tau-1}}(x_\tau,a_\tau)^\top. \label{eq:neuralucb-z}
\end{flalign}

We next bound the differences between ROFU and NeuralUCB, and then show that the differences won't affect the regret bound on the order of $\sqrt{T}$.

\textbf{Bound the approximations in above derivation}

We first introduce some existing bounds on how the function value is approximated by the first order Taylor expansion in the NTK regime. 

\begin{Lem}[Lemma 4.1\& Theorem B.3, \cite{cao2019generalization}]\label{lem:first-order-approx}
There exist constants $c_1,c_2,c_3>0$ such that for any $\delta\in(0,1)$, if $\omega$ satisfies that $c_1 (mL)^{-3/2}\left(\log (T|A|L^2/\delta)\right)^{3/2} \le \omega \le c_2 L^{-6} (\log m)^{-3/2}$, then with probability at least $1-\delta$, for all $\theta, \tilde{\theta}$ with $\|\tilde{\theta}-\theta_0\|_2 \le \omega, \|\theta-\theta_0\|_2 \le \omega$, and for all  $x\in \{x_t\}_{t\in [T]}$ and all $a\in A$, we have 
\begin{flalign}
&\left|f_{\tilde{\theta}}(x,a) -f_\theta(x,a) - \langle h_\theta(x,a),  \tilde{\theta} - \theta \rangle\right| \le c_3 \omega^{4/3} L^3 \sqrt{m\log m},\\
&\|h_\theta(x,a)\|_F \le c_4 \sqrt{mL}.
\end{flalign}
\end{Lem}

We also introduce a bound on the difference between the gradient at $\theta$ and the gradient at initial point $\theta_0$ when $\theta$ is not far from $\theta_0$ as required and satisfied in the NTK regime.
\begin{Lem}[Theorem 5, \cite{allen2019convergence}] \label{lem:gradient-distance}
There exist constants $c_1,c_2,c_3>0$ such that for any $\delta \in (0,1)$, if $\omega$ satisfies that $c_1 (mL)^{-3/2}\max\{\log^{-3/2} m, \log^{3/2} (T|A|/\delta)\}\le \omega \le c_2 L^{-9/2} \log^{-3} m\} \le \omega \le c_2 L^{-6} (\log m)^{-3/2}$, then with probability at least $1-\delta$, for all $\theta$ with $\|\theta-\theta_0\|_2 \le \omega$,and for all  $x\in \{x_t\}_{t\in [T]}$ and all $a\in A$, we have 
\begin{flalign}
\|h_\theta(x,a) - h_{\theta_0}(x,a)\|_2 \le c_3 \sqrt{\log m}\omega^{1/3}L^3 \|h_{\theta_0}(x,a)\|_2.
\end{flalign}
\end{Lem}

With these existing bounds, we next show that $\mZ_t$ (Eq.~(\ref{eq:Zt-1})) enjoys the same property as $\tilde{\mZ}_t$ (Eq.~(\ref{eq:neuralucb-z})) as follows. We introduce another sequence $\bar{\mZ}_t = \lambda \mI + \frac{1}{m} \sum_{\tau =1}^t h_{\theta_0}(x_\tau,a_\tau)h_{\theta_0}(x_\tau, a_\tau)^\top$. In \cite[Lemma B.3]{zhou2020neural}, they showed that $\tilde{\mZ}_t$ is close to $\bar{\mZ}_t$. Similarly, we have the following lemma on the distance on $\|\mZ_t - \bar{\mZ}_t\|_F$.

\begin{Lem}\label{lem:Zt-property}
With probabililty $1-\delta$ where $\delta\in (0,1)$, for any $t\in [T]$, we have 
\begin{flalign}
\|\mZ_t\|_2 \le \lambda + c_1 t L, \quad \|\mZ_t -\bar{\mZ}_t\|_F \le c_2 m^{-1/6} \sqrt{\log m} L^4 t^{7/6} \lambda^{-1/6},
\end{flalign}
if the network width $m$ satisfies that
\begin{flalign}
c_3 m^{-3/2} L^{-3/2}[\log (T|A|L^2/\delta)]^{3/2} \le 2 \sqrt{t/(m\lambda)} \le c_4 L^{-6} [\log m]^{-3/2}, \forall t \in [T],
\end{flalign}
with some constants $c_1,c_2, c_3, c_4>0$.
\end{Lem}
\begin{proof}
The proof can be adapted from that of Lemma B.3 in \cite{zhou2020neural} with the fact that $\|h_{\theta_{t-1}}(x,a) - h_{\theta_0}(x,a)\|_2$ is bounded because of Lemma \ref{lem:gradient-distance}.
\end{proof}

We note that Lemma \ref{lem:Zt-property} plays the same role as Lemma B.3 in \cite{zhou2020neural}. Hence with Lemma \ref{lem:Zt-property} we can establish a similar result as Lemma 5.2 in \cite{zhou2020neural} that with high probability, the optimal $\theta^*$ lies in the sequence of confidence sets, i.e., an ellipsoid defined with the new  matrix $\mZ_t$.

Next we show that the first-order Taylor approximation in Eq. (\ref{eq:ntk-hat-theta2}) does not affect the achievable regret bound either. By the proof of Lemma 5.3 in \cite{zhou2020neural},    in order to prove a similar result, we only need to bound 
\begin{flalign}
|\text{OFU}^R(x_t,a) - \tilde{U}_{t,a}|,
\end{flalign}
where $\tilde{U}_{t,a} := \langle h_{\theta_{t-1}}(x_t,a),\theta_{t-1}-\theta_0\rangle + \gamma_{t-1} \sqrt{h_{\theta_{t-1}}(x_t,a)^\top \mZ_{t-1}^{-1} h_{\theta_{t-1}}(x_t,a)/m}$ and $\gamma_{t-1}$ is given by Algorithm 1 in \cite{zhou2020neural}.
The derivation is as follows,
\begin{flalign}
&|\text{OFU}^R(x_t,a) - \tilde{U}_{t,a}| \nonumber\\
&\le |f_{\theta_{t-1}}(x_t,a) -  \langle h_{\theta_{t-1}}(x_t,a),\theta_{t-1}-\theta_0\rangle| \nonumber\\
&\quad+ \left|\sqrt{f_{\hat{\theta}(x_t,a)} (x_t,a)-f_{\theta_{t-1}} (x_t,a)} -\gamma_{t-1}\sqrt{h_{\theta_{t-1}}(x_t,a)^\top \mZ_{t-1}^{-1} h_{\theta_{t-1}}(x_t,a)} \right|\nonumber\\
&= \left|f_{\hat{\theta}(x_t,a)}(x_t,a) - f_{\theta_0}(x_t,a) f_{\theta_{t-1}}(x_t,a) -  \langle h_{\theta_{t-1}}(x_t,a),\hat{\theta}(x_t,a) -\theta_{t-1}\rangle\right| \nonumber\\
&\quad + \sqrt{\left|f_{\hat{\theta}(x_t,a)} (x_t,a)-f_{\theta_{t-1}} (x_t,a) -\gamma_{t-1}^2{h_{\theta_{t-1}}(x_t,a)^\top \mZ_{t-1}^{-1} h_{\theta_{t-1}}(x_t,a)}\right|} \nonumber\\
& \le c_3 m^{-1/6} \sqrt{\log m} t^{2/3} \lambda^{-2/3} L^3+ \sqrt{c_3 m^{-1/6} \sqrt{\log m} t^{2/3} \lambda^{-2/3} L^3},
\end{flalign}
where $c_3$ is a constant and the last inequality is due to Lemma \ref{lem:first-order-approx} and the choice of $\eta$ such that $1/(2\eta) = \gamma_{t-1}^2$. Hence we show the $\text{OFU}^R(x_t,a)$ shares the same property as the $U_{t,a}$ in Algorithm 1 of \cite{zhou2020neural}. The left proof follows a similar procedure.

%% file: appendix/hyperparameter.tex
\section{Experimental details}\label{app:detail}

\textbf{Hyperparameters}: We tune the hyper-parameters of ROFU on statlog and directly apply the hyperparameters on statlog to other datasets except mushroom. This is because the reward scale of mushroom is much larger than other datasets. For baselines, we directly use the best  reported hyper-parameters.

For convenience, we suppose context $x_t$ is sampled from some distribution  $\mathcal{P}$ for all $t\leq T$. And we define a virtual dataset $\{(x'_t)\}_T, x'_t\sim \mathcal{P}$ as well as a parameter trained by minimizing MSE on the virtual dataset $\theta'=\text{MinimizeMSE}\left( \sum_{t\leq T} \sum_{a\in A}(f_{\theta}(x'_t,a) - r(x'_t,a))^2\right)$
 where MinimizeMSE is a gradient-based optimizer, e.g., SGD  or Adam \cite{kingma2014adam}. 

To alleviate the influence of regret injected by neural network training and generalization error, we decompose the regret into two parts:
\begin{align}
    \mathcal{R}_T = \underbrace{\sum_{t\leq T}\max_a r(x_t,a) - \mathbb{E}\sum_{t\leq T}r(x_t,a'_t)}_{\RNum{1}}  + \underbrace{\mathbb{E}\sum_{t\leq T}r(x_t,a'_t) - \mathbb{E}\sum_{t\leq T}r(x_t, a_t)}_{\text{\RNum{2}}}
\end{align}

where $a'_t=\argmax_{a}f_{\theta'}(x_t,a)$ and $a_t$ is selected by agent. 

It is noteworthy that in practice, $\mathbb{E}\sum_t r(x_t,a'_t)$ is usually an upper bound of $\mathbb{E}\sum_{t\leq T}r(x_t, a_t)$ as $\theta'$ is trained on $T\times |A|$ data points while $\Bar{\theta}$ and $\hat{\theta}$ are trained  on no more than $T$ data points. 
And  the regret  \RNum{1} is caused by MinimizeMSE and generalization power of function $f_\theta$. So, regret \RNum{1} is independent with the algorithm for EE trade-off. Thus, we only  report regret \RNum{2} in Fig. \ref{fig:ncb}  to give a clearer evaluation on ROFU's ability for EE trade-off. We also present the full regret $R_T$ in Fig. \ref{fig:ncb_full}. The running time is in Table. \ref{table:final_running_time}. 

\begin{table}[H]\label{table:nonlinear-cb-results-table}
\begin{center}
\begin{adjustbox}{width=1.\textwidth} 
	\begin{tabular}{c c c  c  c  c c  c  }
	    \Xhline{2\arrayrulewidth}
     		  &\small ROFU & \small Neural-UCB  & \small Dropout & \small ParamNoise & \small Greedy & \small Bootstrap &\small Neural Linear  \\
     	\hline
     	   Runing time (s)&5259 &1589 & 525 & 306 & 771  & 729 & 3612 \\
        \hline
	\end{tabular}
\end{adjustbox}
\end{center}
\caption{Running time of each algorithm for $20000$ rounds. All experiments are run on Nvidia Tesla P-100 GPU. Our algorithm is slower than NeuralUCB because NeuralUCB uses the diagonal approximation which lacks  theoretical guarantee for the regret.}
\label{table:final_running_time}
\end{table}

 \begin{figure*} [t!]
\centering
\subfigure{\label{fig:covertype}\includegraphics[width=1.\textwidth]{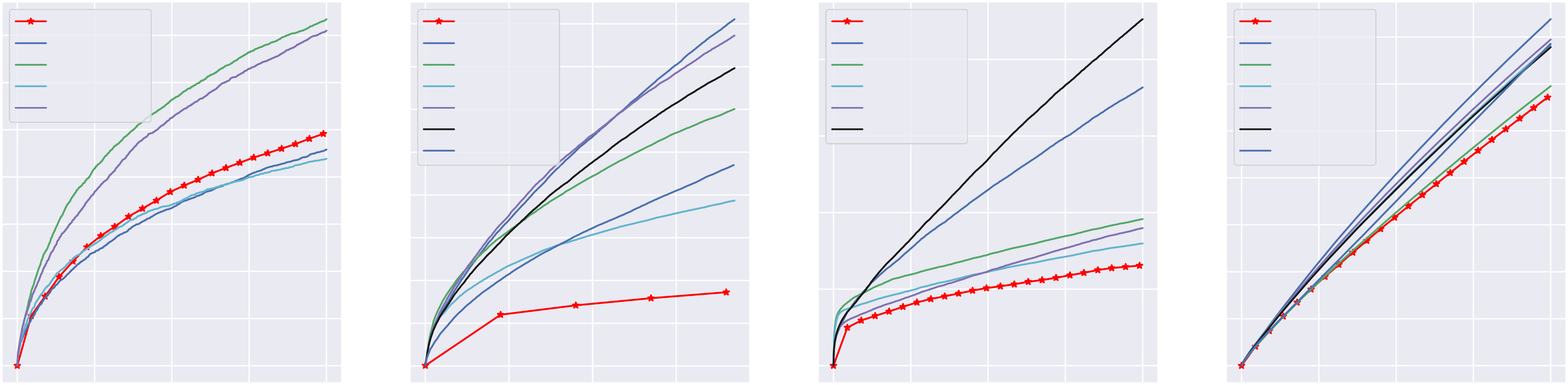}
}
 \quad
 \subfigure{\label{fig:covertype}\includegraphics[width=.75\textwidth]{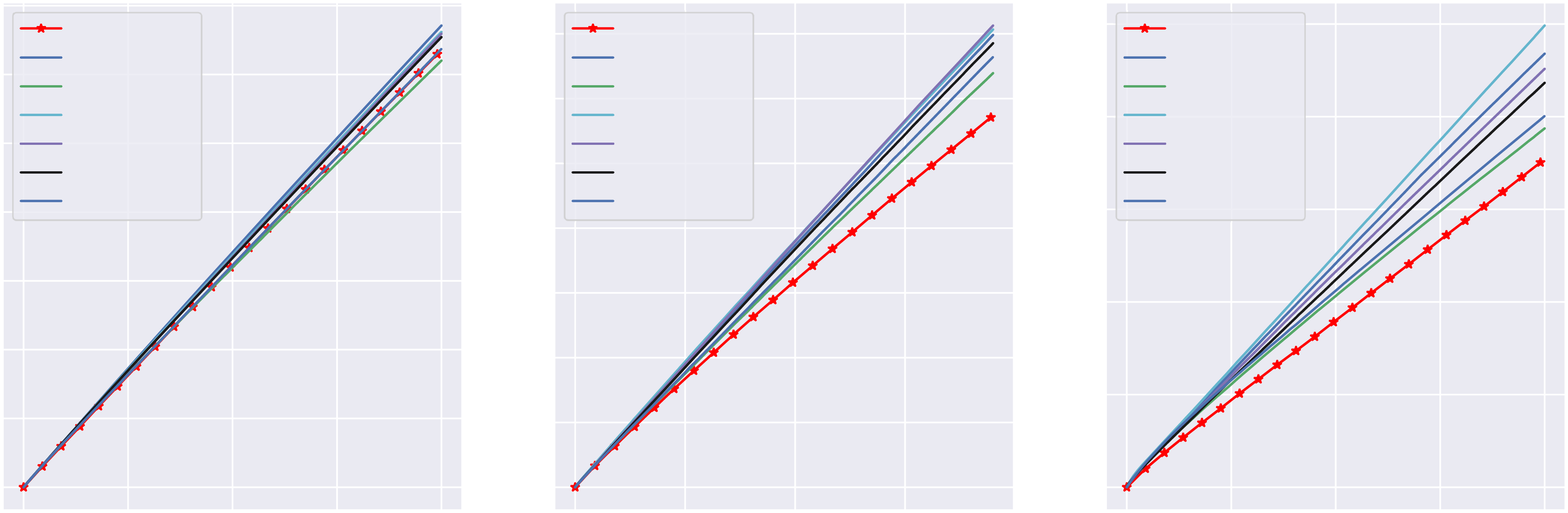}
}

\caption{Evaluations on non-linear contextual bandits.   }
 \label{fig:ncb_full}
\end{figure*}